\documentclass[letterpaper, 10 pt, conference]{ieeeconf}  

\IEEEoverridecommandlockouts{}                              
\usepackage{fix-cm}
\usepackage{etex}

\usepackage{dblfloatfix}

\usepackage{nag}

\makeatletter
\@ifpackageloaded{xcolor}{}{%
\usepackage[table,x11names,dvipsnames,svgnames]{xcolor}%
}
\makeatother

\usepackage{colortbl}

\usepackage{graphicx}
\usepackage{wrapfig}

\definecolorset{RGB}{lyft}{}{Red,194,39,36;Sunset,202,53,33;Orange,205,68,20;Amber,200,117,42;Yellow,242,169,52;Citron,186,188,44;Lime,112,159,33;Green,56,139,31;Mint,45,118,56;Teal,52,133,135;Cyan,60,132,202;Blue,55,94,248;Indigo,64,13,247;Purple,115,42,248;Pink,176,25,145;Rose,176,32,75}

\usepackage{cite}

\usepackage{microtype}

\usepackage[american]{babel}

\usepackage{array}
\usepackage{multirow}
\usepackage{booktabs}
\usepackage{makecell} 

\newcommand{\myparagraph}[1]{\textbf{\emph{#1}}.}

\ifcsname labelindent\endcsname
\let\labelindent\relax
\fi
\usepackage[inline]{enumitem}

\usepackage{subfig}

\newenvironment{lenumerate}[2][]
{\begin{enumerate}[label=(#2\arabic*),leftmargin=0.2in,itemindent=0.15in,#1]}
{\end{enumerate}}

\setlist*[enumerate,1]{label={\itshape\arabic*)}}

\makeatletter
\newcommand{\paragraphswithstop}{%
\let\copyparagraph\paragraph%
\renewcommand\paragraph[1]{\copyparagraph{##1.}}%
}
\makeatother

\usepackage[framemethod=tikz]{mdframed}

\makeatletter
\def\namedlabel#1#2{\begingroup
  #2%
  \def\@currentlabel{#2}%
  \phantomsection\label{#1}\endgroup
}
\makeatother

\usepackage{suffix}

\usepackage{environ}

\makeatletter
\newsavebox{\boxifnotempty}
\newcommand{\displayifnotempty}[3]{\sbox\boxifnotempty{#2}\setbox0=\hbox{\usebox{\boxifnotempty}\unskip}%
\ifdim\wd0=0pt
\else
 #1\usebox{\boxifnotempty}#3%
\fi%
}
\newcommand{\ifempty}[2]{\setbox0=\hbox{#1\unskip}%
\ifdim\wd0=0pt%
 #2%
\fi%
}
\newcommand{\ifnotempty}[2]{\setbox0=\hbox{#1\unskip}%
\ifdim\wd0>0pt%
 #2%
\fi%
}
\makeatother

\usepackage{algorithm}

\usepackage{scrlfile}

\makeatletter
\newcommand*\newstoreddef[1]{
  \BeforeClosingMainAux{%
    \immediate\write\@auxout{%
      \string\restoredef{#1}{\csname #1\endcsname}%
    }%
  }%
}
\newcommand*{\restoredef}[2]{
  \expandafter\gdef\csname stored@#1\endcsname{#2}%
}
\newcommand*{\storeddef}[1]{
  \@ifundefined{stored@#1}{0}{\csname stored@#1\endcsname}%
}
\makeatother

\usepackage{pageslts}
\NewEnviron{tee}{\BODY\typeout{Marker Tee [start] ^^J \BODY ^^JMaker Tee [end]}}

\input{preamble/math}
\newcommand{\real}[1]{\mathbb{R}^{#1}{}}

\newcommand{\bmat}[1]{\begin{bmatrix}#1\end{bmatrix}}

\newcommand{\transpose}{^\mathrm{T}}

\newcommand{\inverse}{^{-1}}

\DeclarePairedDelimiter{\norm}{\lVert}{\rVert}

\newcommand{\vct}[1]{\mathbf{#1}}

\DeclareMathOperator{\stack}{stack}

\providecommand{\cC}{\mathcal{C}}

\providecommand{\cE}{\mathcal{E}}

\providecommand{\cI}{\mathcal{I}}

\providecommand{\cL}{\mathcal{L}}

\providecommand{\cO}{\mathcal{O}}

\providecommand{\cT}{\mathcal{T}}
\providecommand{\cU}{\mathcal{U}}
\providecommand{\cV}{\mathcal{V}}

\providecommand{\cX}{\mathcal{X}}
\providecommand{\cY}{\mathcal{Y}}

\usepackage{units}

  \newcommand{\newcolorlabel}[2]{%
  \expandafter\newcommand\csname #1\endcsname[1]{%
    \tikz[baseline]{\node[text=white,fill=#2,anchor=base,text height=1.3ex,text depth=0.1ex,font=\sffamily\bfseries]{##1}}}%
}

\newcommand{\newcommenter}[2]{%
  \expandafter\newcommand\csname #1\endcsname[1]{%
    \fcolorbox{#2}{#2}{\color{white}\textsf{\textbf{#1}}}
    {\color{#2}##1}}%
  \expandafter\newcommand\csname at#1\endcsname{%
    \fcolorbox{#2}{#2}{\color{white}\textsf{\textbf{@#1}}}
    {\color{#2}}}%
  \expandafter\newcommand\csname #1hl\endcsname[2]{%
    \colorbox{#2}{\color{white}\textsf{\textbf{#1}}}\sethlcolor{Azure2}\hl{##2}~%
    \expandafter\ifx\csname commentarrow\endcsname\relax$\leftarrow$\else \commentarrow[#2]\fi~%
    {\color{#2}##1}}%
  \expandafter\newcommand\csname #1st\endcsname[2]{%
    \colorbox{#2}{\color{white}\textsf{\textbf{#1}}}\sout{##2}~%
    \expandafter\ifx\csname commentarrow\endcsname\relax$\leftarrow$\else \commentarrow[#2]\fi~%
    {\color{#2}##1}}%
}
\newcommenter{TODO}{DodgerBlue1}
\newcommenter{rtron}{Green3}

\usepackage{comment}

\usepackage{pdfcomment}

\usepackage{soul}

\usepackage[normalem]{ulem}

\usepackage{csquotes}

\usepackage{tikz}
\usetikzlibrary{calc}
\usetikzlibrary{matrix}
\usetikzlibrary{chains}
\usetikzlibrary{shapes.geometric}
\usetikzlibrary{arrows.meta}
\usetikzlibrary{decorations.pathreplacing}
\usetikzlibrary{backgrounds}

\tikzset{
  dim above/.style={to path={\pgfextra{
        \pgfinterruptpath
        \draw[>=latex,|->|] let
        \p1=($(\tikztostart)!1.5em!90:(\tikztotarget)$),
        \p2=($(\tikztotarget)!1.5em!-90:(\tikztostart)$)
        in(\p1) -- (\p2) node[pos=.5,sloped,above]{#1};
        \endpgfinterruptpath
      }
    }
  },
  dim double above/.style={to path={\pgfextra{
        \pgfinterruptpath
        \draw[>=latex,|->|] let
        \p1=($(\tikztostart)!3em!90:(\tikztotarget)$),
        \p2=($(\tikztotarget)!3em!-90:(\tikztostart)$)
        in(\p1) -- (\p2) node[pos=.5,sloped,above]{#1};
        \endpgfinterruptpath
      }
    }
  },
  dim below/.style={to path={\pgfextra{
        \pgfinterruptpath
        \draw[>=latex,|->|] let 
        \p1=($(\tikztostart)!-1em!-90:(\tikztotarget)$),
        \p2=($(\tikztotarget)!-1em!90:(\tikztostart)$)
        in (\p1) -- (\p2) node[pos=.5,sloped,below]{#1};
        \endpgfinterruptpath
      }
    }
  },
}

\tikzset{
    right angle quadrant/.code={
        \pgfmathsetmacro\quadranta{{1,1,-1,-1}[#1-1]}     
        \pgfmathsetmacro\quadrantb{{1,-1,-1,1}[#1-1]}},
    right angle quadrant=1, 
    right angle length/.code={\def\rightanglelength{#1}},   
    right angle length=2ex, 
    right angle symbol/.style n args={3}{
        insert path={
            let \p0 = ($(#1)!(#3)!(#2)$) in     
                let \p1 = ($(\p0)!\quadranta*\rightanglelength!(#3)$), 
                \p2 = ($(\p0)!\quadrantb*\rightanglelength!(#2)$) in 
                let \p3 = ($(\p1)+(\p2)-(\p0)$) in  
            (\p1) -- (\p3) -- (\p2)
        }
    }
}

\newcommand{\pgfextractangle}[3]{%
    \pgfmathanglebetweenpoints{\pgfpointanchor{#2}{center}}
                              {\pgfpointanchor{#3}{center}}
    \global\let#1\pgfmathresult  
}

\usetikzlibrary{shapes.arrows}
\newcommand{\commentarrow}[1][Azure4]{\tikz[baseline=-3pt]{\node[shape border uses incircle, fill=#1,rotate=180,single arrow, inner sep=1pt, minimum size=6pt, single arrow head extend=2pt]{};}}

\tikzset{ax/.style={-latex,line width=2pt}}

\tikzset{camera/.style={fill=Sienna1,fill opacity=0.5},%
image plane/.style={draw=RoyalBlue3,line width=2pt}}

\usepackage{tkz-euclide}
\newcommenter{mmitjans}{Emerald}
\newcommenter{mahroo}{BurntOrange}
\newcommand{\prm}{{\texttt{PRM}}}
\newcommand{\rrt}{{\texttt{RRT$^*$}}}
\newcommand{\rrtstar}{{\texttt{RRT$^*$}}}

\usepackage[dvipsnames]{xcolor}
\usepackage{subfig}
\usepackage{xurl}

\newsavebox{\bigleftbox}

\title{\LARGE \bf
  Sample-Based Output-Feedback Navigation with Bearing Measurements
}

\author{Mahroo Bahreinian$^{1}$, Marc Mitjans$^{2}$, Roy Xing$^{3}$ and Roberto Tron$^{4}$
  \thanks{This work was supported by ONR MURI N00014-19-1-2571 ``Neuro-Autonomy: Neuroscience-Inspired Perception, Navigation, and Spatial Awareness''}
  \thanks{$^{1}$Mahroo Bahreinian is with Division of Systems Engineering at Boston University, Boston, MA, 02215 USA. Email:
    {\tt\small mahroobh@bu.edu}}%
  \thanks{$^{2}$Marc Mitjans is with Department of Mechanical Engineering at Boston University, Boston, MA, 02215 USA. M. Mitjans is additionally supported by ”la Caixa” Foundation fellowship LCF/BQ/AA18/11680117.
    {\tt\small mmitjans@bu.edu}}%
    \thanks{$^{3}$ Roy Xing is with Department of Mechanical Engineering at Boston University, Boston, MA, 02215 USA.
    {\tt\small royx@bu.edu}}%
  \thanks{$^{4}$Roberto Tron is with Department of Mechanical and System Engineering.{\tt\small tron@bu.edu}}%
}
\usepackage{algorithm}

\usepackage{algpseudocode}

\begin{document}

\maketitle
\thispagestyle{empty}
\pagestyle{empty}

\begin{abstract}
  We consider the problem of sample-based feedback-based motion planning from \emph{bearing} (direction-only) measurements. We build on our previous work that defines a cell decomposition of the environment using \rrtstar{}, and finds an output feedback controller to navigate through each cell toward a goal location using duality, Control Lyapunov and Barrier Functions (CLF, CBF), and Linear Programming. In this paper, we propose a novel strategy that uses relative bearing measurements with respect to a set of landmarks in the environment, as opposed to full relative displacements. The main advantage is then that the measurements can be obtained using a simple monocular camera.
We test the proposed algorithm in the simulation, and then in an experimental environment to evaluate the performance of our approach with respect to practical issues such as mismatches in the dynamical model of the robot, and measurements acquired with a camera with a limited~field~of~view.
\end{abstract}

\section{INTRODUCTION}
Motion planning is a major research area in the context of mobile robots, as it deals with the problem of finding a path from an initial state toward a goal state while avoiding collisions \cite{laumond1998robot,latombe2012robot,choset2005principles}. Traditional path planning methods focus on finding \emph{single nominal paths} in a given \emph{known map}, and the majority of them makes the implicit assumption that the agent possesses a lower-level \emph{state feedback} controller for following such nominal path in the face of external disturbances and imperfect models. In this paper, we instead use the alternative approach of synthesizing a set of output-feedback controllers instead of single a nominal path. Focusing on controllers allows us to directly consider the measurements (outputs) available to the agent, instead of assuming full state knowledge (i.e., a perfect localization in the environment); the main advantage of this approach is that it is intrinsically robust to disturbances and limited changes in the environment. The main challenge, however, is that we need to explicitly take into account the type of measurements that are available. For instance, in this paper we consider the case of a mobile robot that can use a monocular camera to detect a set of landmarks (e.g., objects) in the environment. 
Using image-based information, the robot can compute the bearing (i.e., the direction vector) with respect to the landmarks. 
The goal of this work is then to synthesize controllers that can use this type of measurements for feedback-based path planning. 

\myparagraph{Previous works}
A well known class of techniques for motion planning is sampling-based methods. Algorithms such as Probabilistic RoadMaps (\prm) \cite{kavraki1996probabilistic}, Rapidly exploring Random Trees (\rrt, \cite{rrt,lavalle2001randomized}) and asymptotically optimal Rapidly Exploring Random Tree (\rrtstar{}, \cite{karaman2011sampling}), have become popular in the last few years due to their good practical performance, and their probabilistic completeness \cite{lavalle2006planning,lavalle2001randomized,karaman2011sampling}; there have also been extensions considering perception uncertainty \cite{renganathan2020towards}. However, these algorithms only provide nominal paths, and assume that a separate low-level controller exists to generate collision-free trajectories at run time.

The concept of the sample-based output feedback controller is based on \cite{Mahroo,Mahroo2} which, the output feedback controller takes as input relative displacement measurements with respect to a set of landmarks. This algorithm is based on solving a sequence of robust min-max Linear Programming (LP) problems on the elements of a convex cell decomposition of the environment and using linear Control Lyapunov Function (CLF) and Control Barrier Function (CBF) constraints, to provide stability and safety guarantees, respectively. Additionally, it extends planning beyond the discrete nominal paths, as feedback controllers can correct deviations from such paths, and are robust to discrepancies between the planning and real environment maps. The authors in \cite{Mahroo} assumes that a polyhedral convex cell decomposition of the environment is available, which greatly reduces its applicability. The authors in \cite{Mahroo2} extended upon \cite{Mahroo} and introduces a representation of the environment that is obtained via a modified version of the optimal Rapidly-exploring Random Trees (\rrtstar{}), combined with a convex cell decomposition of the environment.


In order to observe the environment, modern robots can be equipped with different kinds of sensors, such as Inertial Measurement Units \cite{zhao2011motion}, visual sensors \cite{papanikolopoulos1993visual}, sonar sensors\cite{yata1999fast}, and contact sensors~\cite{lumelsky1990unified}, etc. Among these, monocular cameras are nearly ubiquitous, given their favorable trade-offs with respect to size, power consumption, and the richness of information in the measurements; a peculiarity of this sensor is that, due to perspective projection, it can provide relatively precise \emph{bearing} information (the direction of an object or point in the environment) but not, from a single image alone, \emph{depth} information. Many different techniques have been introduced for monocular navigation and control.
An example is visual servoing, where image-based computations are used to find a control law that guides the robot toward a \textit{home} location using bearing alone \cite{lambrinos1998landmark,argyros2001robot,liu2010bearing,tron2014optimization}, or by concurrently estimating depths \cite{pentland1989simple,ens1993investigation,michels2005high}.

\myparagraph{Proposed approach and contributions}
The goal of the present work is to joint the two threads of work mentioned above: apply the synthesis method of \cite{Mahroo,Mahroo2} but with controllers that use as input a set of bearing measurements between the robot and a set of landmarks.
Since we consider a planning problem, we assume that a map of the environment is available, and that the positions of the landmarks are known and fixed in the reference frame of this map; furthermore, we assume that a global compass direction is available (so that the orientation, but not the position of the robot in the map is known).
The main contribution of this work is then to provide a way to synthesize bearing-based feedback controllers that can be used to solve the path planning problem by extending the method of~\cite{Mahroo,Mahroo2} via a dynamic rescaling of the measurements. This contribution is tested in both simulations and experiments with real robots.

\section{BACKGROUND}\label{sec:background}
In this section, we review the system dynamics, bearing direction measurements, CLF and CBF constraints, and the \rrtstar{} method in the context of our proposed work.
\subsection{System Dynamics}
Consider a driftless control affine system
\begin{equation}\label{sys1}
  \dot{x}= Bu,
\end{equation}
where $x \in \cX\subset\real{n}$ denotes the state, $u\in\cU\subset\real{m}{}$ is the system input, and $B\in\real{n\times m}$ define the linear dynamics of the system. 
Note that 
We assume that the system \eqref{sys1} is controllable, that $\cX$ and $\cU$ are polytopic,
\begin{align}\label{state_limits}
  \cX=\{x\mid A_{x}x\leq b_{x}\},&& \cU=\{u\mid A_{u}u\leq b_u\},
\end{align}
and that $0\in\cU$.
The assumptions above are more restrictive than those in, e.g., \cite{Mahroo2}, but are necessary to allow our solution to the bearing-based control problem.

Note that, in our case, $\cX$ will be a convex cell centered around a sample in the tree (Sec.~\ref{sec:simplified-tree}).
\subsection{Bearing direction measurements}\label{sec:bearing}
In this work, we assume global compass direction is available (the bearing directions can
be compared in the same frame of reference) and the robot has only access to bearing direction measurements (see Fig.~\ref{bearing}). The bearing direction measurements are defined as
\begin{equation}\label{bearing_mes}
    \beta_i= d_i(x)^{-1}(l_i-x),\;\;\;\; i=\{1,\hdots,N\},
\end{equation}
where $N$ is the number of \emph{landmarks} and $d_i$ is the relative displacement measurement between the position of the robot and landmark $l_i$
\begin{equation}\label{distance}
    d_i(x) = \norm{l_i-x}.
\end{equation}
The location of landmarks ${l}_i$ is assumed to be fixed and known at the planning time. Additionally, we assume at the implementation time, the robot can measure the bearing direction $\beta_i$ between its position $x$ and a set of landmarks $l_i$. 

\begin{figure*}
\subfloat[]{\label{bearing}{\includegraphics[width=5.5cm]{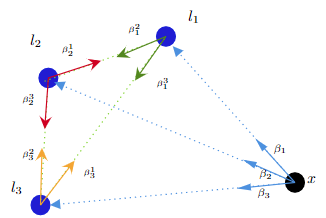}}}
\hfill
 \subfloat[]{\label{circle}{\includegraphics[width=5.5cm]{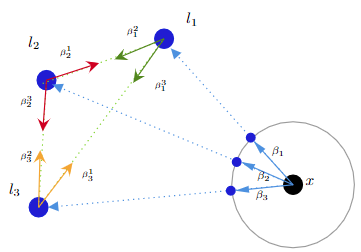}}}
 \hfill
  \subfloat[]{\label{config}{\includegraphics[width=4.5cm]{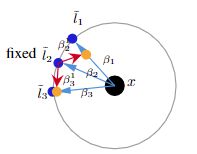}}}
  \caption{a) Bearing direction measurements. Landmarks are shown by blue circles and the robot is shown by a black circle. We assume the robot measures the bearing direct to landmarks and they are shown by blue arrows. The position of landmarks is known and bearing direction measurements between landmarks are shown by purple, green, and orange arrows. b) As bearing measurements are unit vectors, the robot sees all the landmarks within 1 unit from itself. c) modify bearing measurements by assuming the landmark $l_2$ is fixed.}
\end{figure*}

\subsection{Optimal Rapidly-Exploring Random Tree (\rrtstar{})}
\label{sec:rrtstar}
The \rrtstar{} sampling-based algorithm is typically used for single-query path planning, but that can also be used (as in this paper) to build a tree $\cT=(\cV,\cE)$ representing the free configuration space starting from a given root node. In this work, we build a tree in the free configuration space by implementing the \rrtstar{} algorithm, and the only modification we made to the original \rrtstar{} algorithm is that we store random samples that were found to be in collision with an obstacle in the list $\cV_{\text{collision}}$ instead of discarding them; this list is then returned and used to define the CBF constraints in our algorithm (see Sec.~\ref{sec:tree-CBF}).
\rrtstar{} is guaranteed to be asymptotically complete and optimal, although these guarantees do not hold with a finite number of samples.

\subsection{Convex Decomposition of the Environment by Simplified \rrtstar{} Trees}\label{sec:simplified-tree}
We start with a tree $\cT = (\cV,\cE)$ generated by the traditional \rrtstar{} algorithm from Sec. \ref{sec:rrtstar}. Since the number of samples is finite, the generated tree is not optimal although it has a large number of nodes. We simplify the tree to reduce the number of nodes (while keeping all the samples that are in collision with obstacles) by following the simplified-\rrtstar{} algorithm in \cite{Mahroo2} and denote it as $\cT=(\cV_s,\cE_s)$.

Note that as a consequence of the simplifying steps above, it is still possible to connect any sample that was discarded from the original \rrtstar{} to the simplified tree with a straight line, suggests that the simplified tree will be a good road-map representation \cite{Choset:book05} of the free configuration space reachable from the root (up to the effective resolution given by the original sampling).

Given the simplified tree $\cT_s=(\cV_s,\cE_s)$, for each node $i\in \cV_s$ in the tree, we define a convex cell $\cX_{ij}$ similar to \cite{Mahroo2} such that the boundaries of $\cX_{ij}$ are defined as the bisectors hyper-plane between node $i$ and other nodes in the tree except node $j$ which node $j$ is the parent of node $i$.
The polyhedron $\cX_{ij}$ is similar to a Voronoi region \cite{latombe2012robot}(See Fig.~\ref{fig:CBF-X} for an example). Note that $\cX_{ij}$ contains all the points that are closest to $i$ than other vertices in $\cT_s$, but it also includes the parent~$j$.
\begin{figure}[t]
  \centering
  \includegraphics[width=5.5cm]{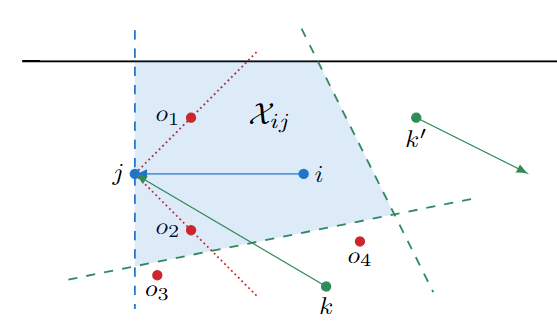}
 
  \caption{The Voronoi-like region $\cX_{ij}$ for edge $(i,j)$, and the corresponding CBF constraints; green points and arrows: other vertices and edges in the tree $\cT_s$; red points: collision samples in $\cV_{\text{collision}}$ that are inside $\cX_{ij}$; black line: boundaries of the configuration space; dashed lines: boundaries of the convex set $\cX_{ij}$; red dotted lines: CBF constraints $h_i$. }\label{fig:CBF-X}
\end{figure}
The environment is implicitly represented by a simplified tree graph $\cT_s=(\cV_s,\cE_s)$ via sampling. According to the simplified tree, the environment is decomposed the environment into a set of convex cells.


Note that the landmarks $\hat{l}_i$, from the point of view of our algorithms, can be arbitrary as long as there is at least two landmark visible from any point $x$ in the free configuration space. The landmarks do not need to be chosen from the samples of the \rrtstar{} algorithm, or from the obstacles.


\subsection{Control Lyapunov and Barrier Functions (CLF, CBF)}\label{sec:ECBF}
In this section, we review the CLF and CBF constraints, which are differential inequalities that ensure stability and safety (set invariance) of a control signal $u$ with respect to the dynamics \eqref{sys1}. First, it is necessary to review the following.

\begin{definition}
Given a sufficiently smooth function $h:\real{n}\to\real{}$ and a vector field $f:\real{n}\to\real{n}$, we use the notation $L_fh=\nabla_x h\transpose f$ to denote the Lie derivative of $h$ along $f$, where $\nabla h$ represents the gradient field of $h$.
\end{definition}
The Lie derivative of a differentiable function $h(x)$ for the dynamics \eqref{sys1} with respect to the vector field $B$ is defined as $\cL_{B}h(x)$. Applying this definition to \eqref{sys1} we obtain
\begin{equation}\label{Lie}
  \cL_Bh(x) = \frac{\partial h(x)}{\partial x} Bu.
\end{equation}

In this work, we assume that Lie derivatives of $h(x)$ of the first order are sufficient \cite{Isidori:book95} (i.e., $h(x)$ has relative degree $1$ with respect to the dynamics \eqref{sys1}); however, the result could be extended to the higher degrees, as discussed in \cite{Mahroo}.

We now pass to the definition of the differential constraints.
Consider a continuously differentiable function $V(x):\cX\to\real{}$, $V(x)\geq 0$ for all $x\in\cX$, with $V(x)=0$ for some $x\in\cX$.
\begin{definition}\label{def:ECLF}
  The function $V(x)$ is a \textit{Control Lyapunov Function} (CLF) with respect to \eqref{sys1} if there exists positive constants $c_1,c_2,c_v$ and control inputs $u\in \cU$ such that
  \begin{equation}\label{cons:clf}
    \begin{aligned}
      \cL_BV(x)u+c_v V(x)\leq 0,\forall x \in \cX.
    \end{aligned}
  \end{equation}
  Furthermore, \eqref{cons:clf} implies that $\lim_{t\to\infty}V(x(t))=0$.
\end{definition}
Consider a continuously differentiable function $h(x):\cX\to\real{}$ which defines a safe set $\cC$ such that
\begin{equation}\label{set_c}
  \begin{aligned}
    \cC&=\{x\in \real{n}|\;h(x)\geq0\},\\
    \partial \cC&=\{x\in \real{n}|\;h(x)=0\},\\
    {Int}(\cC)&=\{x\in \real{n}|\;h(x)>0\}.
  \end{aligned}
\end{equation}
In our setting, the set $\cC$ will represent a convex local approximation of the free configuration space (in the sense that $x\in\cC$ does not contain any sample that was found to be in collision).
We say that the set $\cC$ is \emph{forward invariant} (also said \emph{positive invariant} \cite{positiveInvariant}) if  $x(t_0) \in \cC$ implies $x(t)\in \cC$, for all $t\geq 0$~\cite{zcbf1}.
\begin{definition}[CBF, \cite{nguyen2016exponential}]\label{def:ECBF}
  The function $h(x)$ is a \emph{Control Barrier Function} with respect to \eqref{sys1} if there exists a positive constant $c_h$, control inputs $u\in \cU$, and a set $\cC$ such that
  \begin{equation}\label{cons:cbf}
   \cL_Bh(x)u+c_hh(x)\geq 0,\forall x \in \cC.
  \end{equation}
  Furthermore, \eqref{cons:cbf} implies that the set $\cC$ is forward invariant.
\end{definition}
\subsection{Stability by CLF}\label{sec:tree-CLF}
To stabilize the navigation along an edges of a tree, we define the Lyapunov function $V_{ij}(x)$ as
\begin{equation}\label{V}
  V_{ij}(x)=z_{ij}^T (x-x_{j}),
\end{equation}
where $z_{ij}= \frac{x_j-x_i}{\norm{x_j-x_i}}$ is the unit vector of \text{\emph{exit direction}} for edge $(i,j)$, $x_{j} \in \text{\emph{exit face}}$ is the position of the parent of node $i$, and $V_{ij}(x)$ reaches its minimum $V(x)=0$ at $x_j$. Note that the Lyapunov function represents, up to a constant, the distance $d(x,x_j)$ between the current system position and the exit face. By Definition~\ref{def:ECLF}, $V_{ij}(x)$ is a CLF.


\subsection{Safety by CBF}\label{sec:tree-CBF}

In this section, we define barrier functions $h_{ij}(x)$ that define a cone representing a local convex approximation of the free space between $i$ and $j$, in the sense that it excludes all samples in $\cV$ that are on the way from $i$ to $j$. In particular, we use the following steps:
\begin{enumerate}
\item Define set $\cO_{ij}\subset\cV_{\text{collision}}$ whose are in $\cX_{ij}$.
\item From the set $\cO_{ij}$, we choose two samples $\{o_{u},o_{d}\}$ on two sides of the edge $(i,j)$ such that they are closet samples in $\cO_{ij}$  to edge $(i,j)$.
\item We write the equations of two lines passing through $\{j,o_{u}\}$ and $\{j,o_{d}\}$ in a matrix form using $A_{h_{i}}\in \real{2\times 2}$, $b_h\in \real{2}$ to define the invariant set
  \begin{equation}
    \cC_{ij}=\{x\in\real{2}:A_{h_{ij}}\vct{x}+b_{h_{i}}>0\}.
  \end{equation}
\end{enumerate}
The corresponding CBF is then defined as
\begin{equation}\label{h}
  h_{ij}(x)=A_{h_{ij}}x+b_{h_{ij}}.
\end{equation}
An example of the set $\cC_{ij}$  is shown in Fig.~\ref{fig:CBF-X}. Note that the region $\cC_{ij}$ might not include the entire cell $\cX_{ij}$. However, the controller will be designed to satisfy the CBF and CLF constraints over the entire cell $\cX_{ij}$ \cite{Mahroo2}).

\subsection{Output Feedback Controller}\label{sec:tree-controller}
\newcommand{\xpos}{x_{\textrm{pos}}}
We assume that the robot can only measure the relative displacements between the robot's position $x$ and the landmarks in the environment, which corresponds to the output function
\begin{equation}\label{vec-landmarks}
  \cY=(L-x\vct{1}\transpose)^\vee=L^\vee-\cI x=\stack{(l_i- x)},
\end{equation}
where $L\in\real{n\times n_l}$ is a matrix of landmark locations, $i=1,\hdots,n_l$ that $n_l$ is the number of landmarks,  $A^\vee$ represents the vectorized version of a matrix $A$, $\cI=\vct{1}_{nl} \otimes I_n$, and $\otimes$ is the Kronecker product. We define the feedback controller as
\begin{equation}\label{u}
  u_{ij}=K_{ij}\cY,
\end{equation}
where $K_{ij} \in \real{m\times nn_l}$ are the feedback gains that need to be found for each cell $\cX_{ij}$.

In \eqref{u}, the input to the controller is the relative displacement measurements between the position of the robot and the set of landmarks in the environment. The controller in \eqref{u} is a weighted linear combination of the measured displacements $y$. The goal is to design $u$ such that the system is driven toward the exit direction $z_{ij}$ while avoiding obstacles. 
Assume the controller is in the form of \eqref{u}, and the output function is \eqref{vec-landmarks}, following the approach of \cite{Mahroo,Mahroo2}, and using the CLF-CBF constraints reviewed in Sec.~\ref{sec:background}, we encode our goal in the following feasibility problem:

\begin{equation}\label{findK}
  \begin{aligned}
    & \textrm{find} \;\;{K_{ij}}\\
    & \textrm{subject to:}\\
    &\text{CBF:}\;-( \cL_Bh_{ij}(x)u+c\transpose_hh_{ij}(x))\leq 0,\\
    &\text{CLF:}\;\;\cL_BV_{ij}(x)u+c\transpose_v{V_{ij}}(x)\leq 0,\\
    &u\in\cU,\;\;\forall x\in \cX_{ij},\;\;(i,j)\in\cE.
  \end{aligned}
\end{equation}

The constraints in \eqref{findK} need to be satisfied for all $x$ in the region $\cX_{ij}$, i.e., the same control gains should satisfy the CLF-CBF constraints at every point in the region. In \cite{Mahroo2}, this problem is solved by rewriting \eqref{findK} using a min-max formulation and Linear Programming (LP) method.
Starting from a point $x\in\cX_{ij}$, $u_{ij}$ drives the robot toward $x_j$, the robot switches its controller to  $u_{jq}$ when $x$ reaches the exit face of $\cX_{ij}$.

\section{FEEDBACK CONTROL PLANNING WITH BEARING MEASUREMENTS}\label{sec:problem-setup}


As mentioned in the introduction, in this paper we consider a robot equipped with a monocular camera that, in general, does not provide the depth of a target object in the image, and instead measures the corresponding relative bearing (Sec.~\ref{sec:bearing}).
In this section, we show that in this case, it is possible to still use the control synthesis method of Sec.~\ref{sec:tree-controller} after rescaling the bearing measurements, such that they are similar to the ideal displacement measurements. We divide the section into three parts. First, we give details on the rescaling procedure; then, we show that the resulting bearing controller still solves the path planning problem, albeit with modified CLF and CBF conditions; finally, we discuss how to handle the practical problem of a camera with a limited field of view.

\subsection{Rescaling of the Bearing Measurements}\label{sec:rescaling}
For simplicity, in this section we consider only 2-D environments (however, an extension to the 3-D case is possible with minor modifications). Without loss of generality, we assume that the robot position $x$ is equal to the zero vector (if not, all the computations below are still valid after introducing opportune shifts by $x$).

The bearings $\{\beta_i\}$ can be identified with points on a unitary circle centered at the origin (Fig.~\ref{bearing}). 

For the current cell (and hence its associated controller) we pick a \emph{fixed} landmark $f$ among those available. We then define $\tilde{l}_f$ as the point on the circle corresponding to $\beta_f$. Our goal is to rescale all the other bearings $i\neq f$ such that they are similar to the full displacements.

For this purpose, we define \emph{inter-landmark} bearing directions between landmark $f$ and every other landmarks $i\neq f$ as 
\begin{equation}
    \beta_i^f= d_i(l_f)^{-1}(l_i-l_f);
\end{equation}
note that these bearings can be pre-computed from the known locations of the landmarks in the map.

For each landmark $i\neq f$, let $P_i^f$ be the line passing through the fixed landmark $\tilde{l}_f$ with direction $\beta_i^f$, and let $P_i$ be the line passing through the robot position $x$ with direction $\beta_i$. We then define the scaled landmark position $\tilde{l}_i$ as the intersection between $P_i^f$ and $P_i$.

\begin{lemma}
Let $s(x)=d_f(x)$ be the distance between landmark $f$ and the robot; then, we have that $l_i-x=s(x) (\tilde{l}_i-x)$ for each landmark $i$.
\end{lemma}
\begin{proof}
The triangles $l_f,x,l_i$ and $\tilde{l}_f,x,\tilde{l}_i$ are similar since they have identical internal angles. Moreover, $\norm{\tilde{l}_f-x}=1$ by construction, the ratio between the segments $l_f,x$ and $\tilde{l}_f,x$ is equal to $s(x)$. Combining these two facts, we have that the ratio between the segments $l_f,x$ and $\tilde{l}_f,x$ is also $s(x)$; the claim then follows.
\end{proof}

Our proposed solution is then to compute the scale displacements $\tilde{\cY}=\stack(\{\tilde{l}_i-x)$, which are then used with the pre-computed controller 
\begin{equation}\label{eq:u tilde}
\tilde{u}_{ij}(x)=K_{ij}\tilde{\cY}.
\end{equation}
\subsection{Analysis of the Bearing Controller}%
\label{sec:bearing_measurement}

The following lemma shows that the original displacement-based controller $u_{ij}$ and our proposed bearing-based controller $\tilde{u}_{ij}$ are essentially equivalent from the point of view of path planning.

\begin{proposition}\label{prop:speed reparametrization}
Assume $s(x)$ is uniformly upper bounded (i.e., $s(x)<\infty$ for all $x\in \cX_{ij}$). The controllers $u_{ij}$ in \eqref{u} and $\tilde{u}_{ij}$ in \eqref{eq:u tilde} produce the same paths (but traced, in general, with different speeds) for the drifless system \eqref{sys1} when started from the same initial condition.
\end{proposition}
\begin{proof}
Let $x$ and $\tilde{x}$ be the trajectories of the system under $u_{ij}$ and $\tilde{u}_{ij}$, respectively. Since both the dynamics and the controllers are linear, we have that $\dot{x}=s\dot{\tilde{x}}$ when evaluated at the same location. This implies that the two curves $x$ and $\tilde{x}$ are the same up to a reparametrization of the velocity.
\end{proof}

In fact, we can also relate the new controller to the conditions in the synthesis problem~\eqref{findK}.
\begin{proposition}\label{prop:new CBF CLF conditions}
Assume that $s_{\min}<s(x)\leq s_{\max}$, and that $\tilde{u}=K_{ij}s\cY\in \cU$ for all $x\in \cX_{ij}$. Then $K_{ij}$ is a feasible solution for \eqref{findK} with the modified CLF and CBF conditions:
\begin{align}
-( \cL_Bh_{ij}\tilde{u}+\tilde{c}_hh_{ij})&\leq 0, \\ \cL_BV_{ij}\tilde{u}+\tilde{c}_vV_{ij}&\leq 0,  \end{align}
where $\tilde{c}_h=s_{\min}\inverse c_h$ and $\tilde{c}_v=s_{\max}\inverse c_v$.
\end{proposition}
\begin{proof} The claim follows by dividing the original CBF and CLF conditions by $s(x)$, and then using the bounds $s_{\min}$, $s_{\max}$.
\end{proof}

Note that the bounds on $s(x)$ translate to bounds on the distance between the cell $\cX_{ij}$ and the landmarks $l_i$

Taken together, Propositions~\ref{prop:speed reparametrization} and~\ref{prop:new CBF CLF conditions} show that the controller $K_{ij}$ found by assuming a displacement-based controller can be used also for the bearing-based case, although the speed of the resulting trajectories might be more aggressive.

\subsection{Control With the Limited Field of View}\label{sec:limited-field-of-view}
In the formulation Sec.~\ref{findK} it is implicitly assumed that the controller has access to all the landmarks measurements at all times. However, in practice, a robot will only be able to detect a subset of the landmarks due to a limited field of view or environment occlusions (see Fig.~\ref{FOV}). To tackle this issue, we show in this section that the controller $\tilde{u}$ in \eqref{eq:u tilde} can be designed using all available landmarks (as in the preceding section), but then implemented using two landmarks or more.

The main idea is to use the two available landmarks to estimate the missing landmarks $\tilde{l}_i$ through the following steps (see also Fig.~\ref{LFOV}):
\begin{enumerate}
    \item Choose one landmark $f$ as the fixed landmark, and denote as $f'$ the other landmark (in the figure $f=2$ and $f'=1$). Compute $\beta_{i}^f$ and $\beta_i^{f'}$ for all $i\notin \{f,f'\}$.
    \item Following the same steps as in Sec.~\ref{sec:rescaling}, compute $\tilde{l}_f$ (from the bearing $\beta_f$) and $\tilde{l}_f'$ (via intersection).
    \item Let $P_i^f$ and $P_i^{f'}$ be the lines passing through $l_f$ and $l_{f'}$ in the directions $\beta_i^f$ and $\beta_i^{f'}$, respectively.
    \item Compute the non-visible modified landmark $\tilde{l}_i$ from the intersection of $P_i^f$ and $P_i^{f'}$. 
\end{enumerate}
After finding the modified bearings for all landmarks, the controller is implemented as in~Sec.~\ref{sec:rescaling}.
\begin{figure}
    \centering
    \subfloat[]{\label{LFOV}{\includegraphics[width=5.5cm]{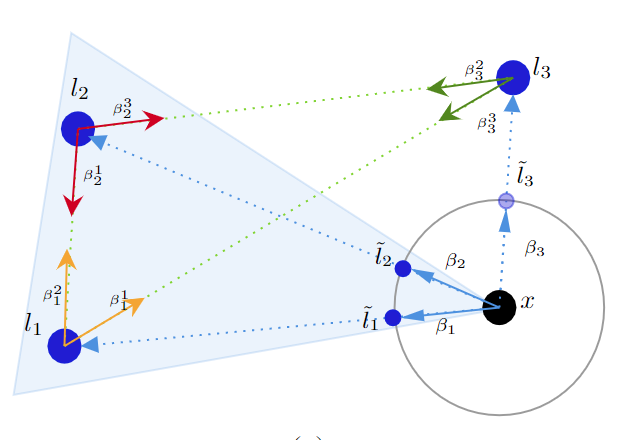}}}
    \hfill
    \subfloat[]{\label{FOV}{\includegraphics[width=2.8cm]{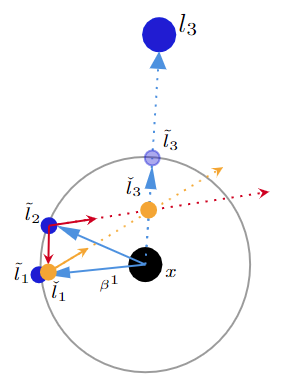}}}
    \caption{Implement controller with limited field of view. In Fig.~\ref{LFOV}, landmark $l_3$ is not visible from the robot camera. In Fig.~\ref{FOV}, the modified observation $l_3$ is computed using the relative bearing measurements between landmarks.}
    \label{fig:my_label}
\end{figure}

\section{SIMULATION AND EXPERIMENTAL RESULTS}
To assess the effectiveness of the proposed algorithm, we run a set of validations using both MATLAB simulations and experiments using ROS on a Jackal robot by Clearpath Robotics~\cite{Clearpath}.
While the optimization problem guarantees exponential convergence of the robot to the stabilization point, in these experiments the velocity control input $u$ has been normalized to achieve constant velocities along the edges of the tree.

\begin{figure}[t]  
  \newcommand{\inpic}[1]{\includegraphics[width=4cm]{#1}}
\centering
  \subfloat[Configuration space]{
    \inpic{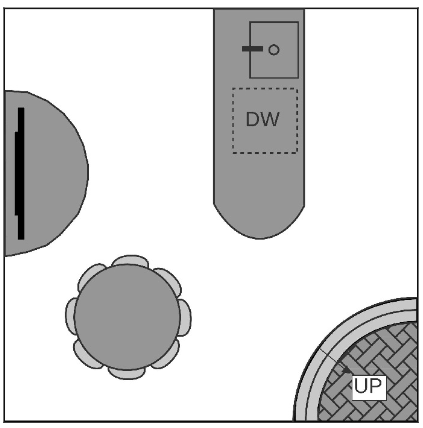}
    \label{fig:Cspace}
  }
  \subfloat[Generated RRT*]{
    \inpic{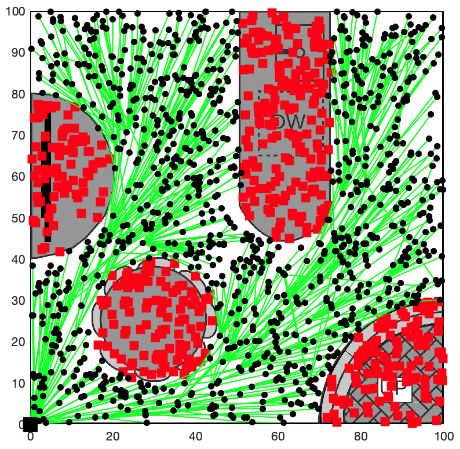}
    \label{fig:RRT}
  }%

  \subfloat[Simplified tree]{
    \inpic{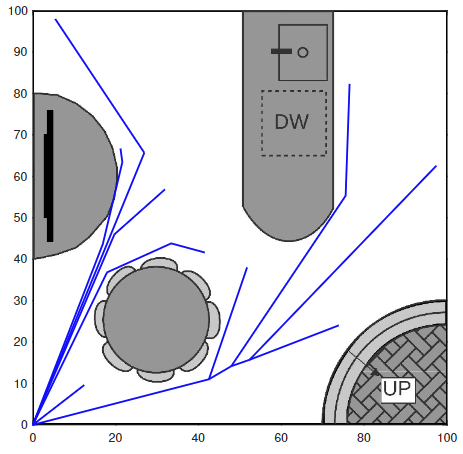}
    \label{fig:RRT_simple}
  }
  \subfloat[Simulation]{
    \inpic{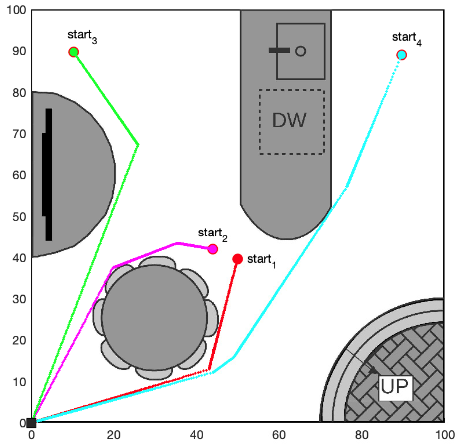}
    \label{fig:Path}
  }%
\caption{Fig.~\ref{fig:Cspace} shows the configuration space of a kitchen-like environment, where obstacles are represented by gray color. In Fig.~\ref{fig:RRT} the root for \rrtstar{} is located at the origin, red dots show the samples in collision with the obstacles, and the tree generated by \rrtstar{} is plotted in green. Fig.~\ref{fig:RRT_simple} depicts the corresponding simplified tree. In Fig.~\ref{fig:Path}, the agent is initialized at different start points and moves toward the goal point.}
\label{fig:exp}
\end{figure}
\subsection{MATLAB Simulation}
The simulated MATLAB environment is presented in Fig.~\ref{fig:exp}, where the obstacles are represented by gray color. For \rrtstar{}, we set the maximum number of iterations in \rrtstar{} to $1500$, we choose $\eta=50$, and we place the root of the tree at the origin of the environment (bottom left corner). The generated tree from \rrtstar{} and its simplified form are shown in Fig.~\ref{fig:RRT} and Fig.~\ref{fig:RRT_simple} respectively. Then, we compute a controller for each edge of the simplified tree as described in Sec.~\ref{sec:tree-controller}. Fig.~\ref{fig:Path} shows the resulting trajectories from four initial positions.
In all cases, the robot reaches the desired goal location by applying the sequence of controllers found during planning.
\begin{figure}[t]
  \centering
  \subfloat[Clearpath Robotics Jackal]{\label{fig:robot}{\includegraphics[width=3.5cm]{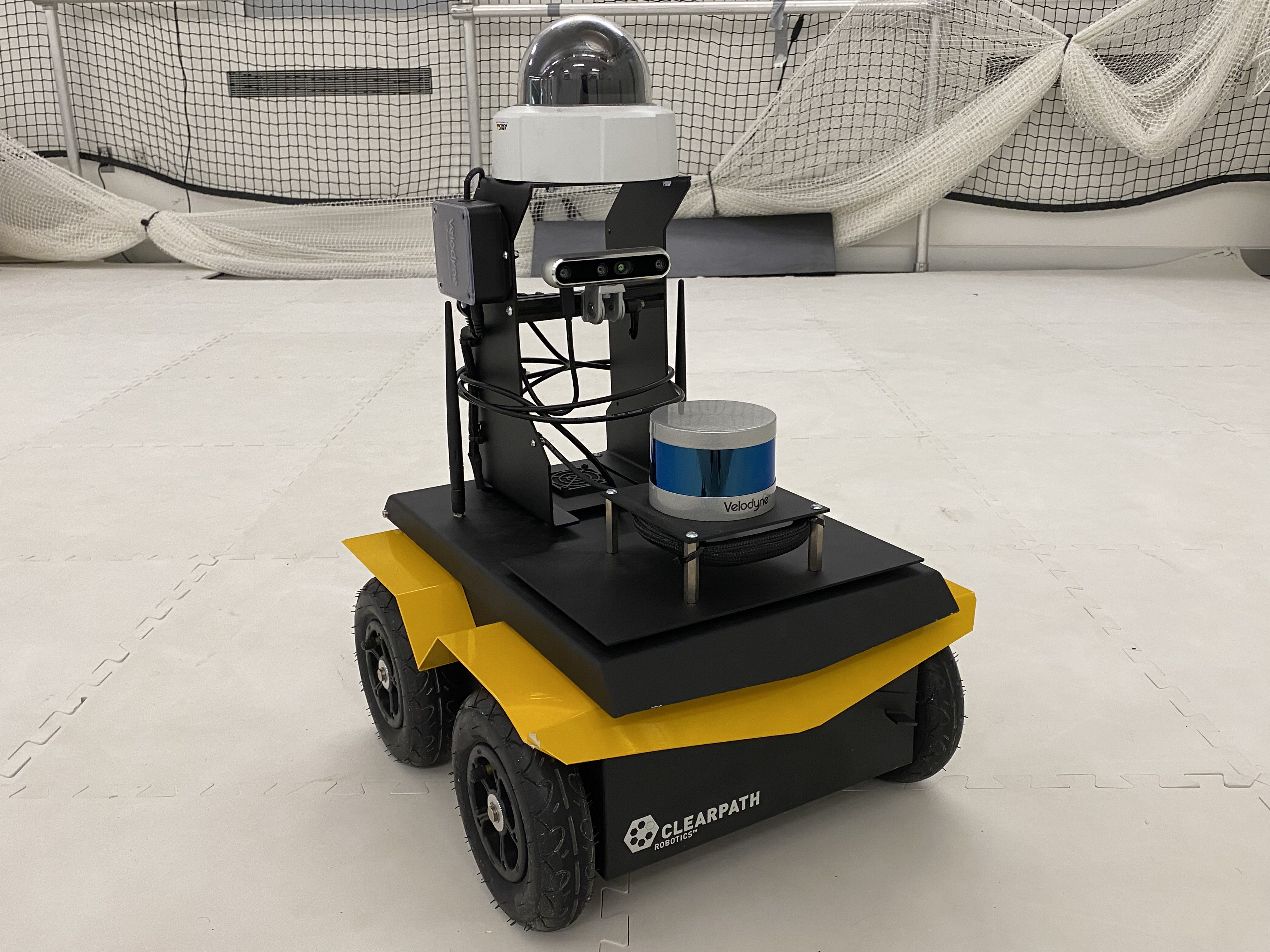} }}
  \subfloat[AprilTag]{{\label{fig:apriltag-figure}\includegraphics[width=2cm]{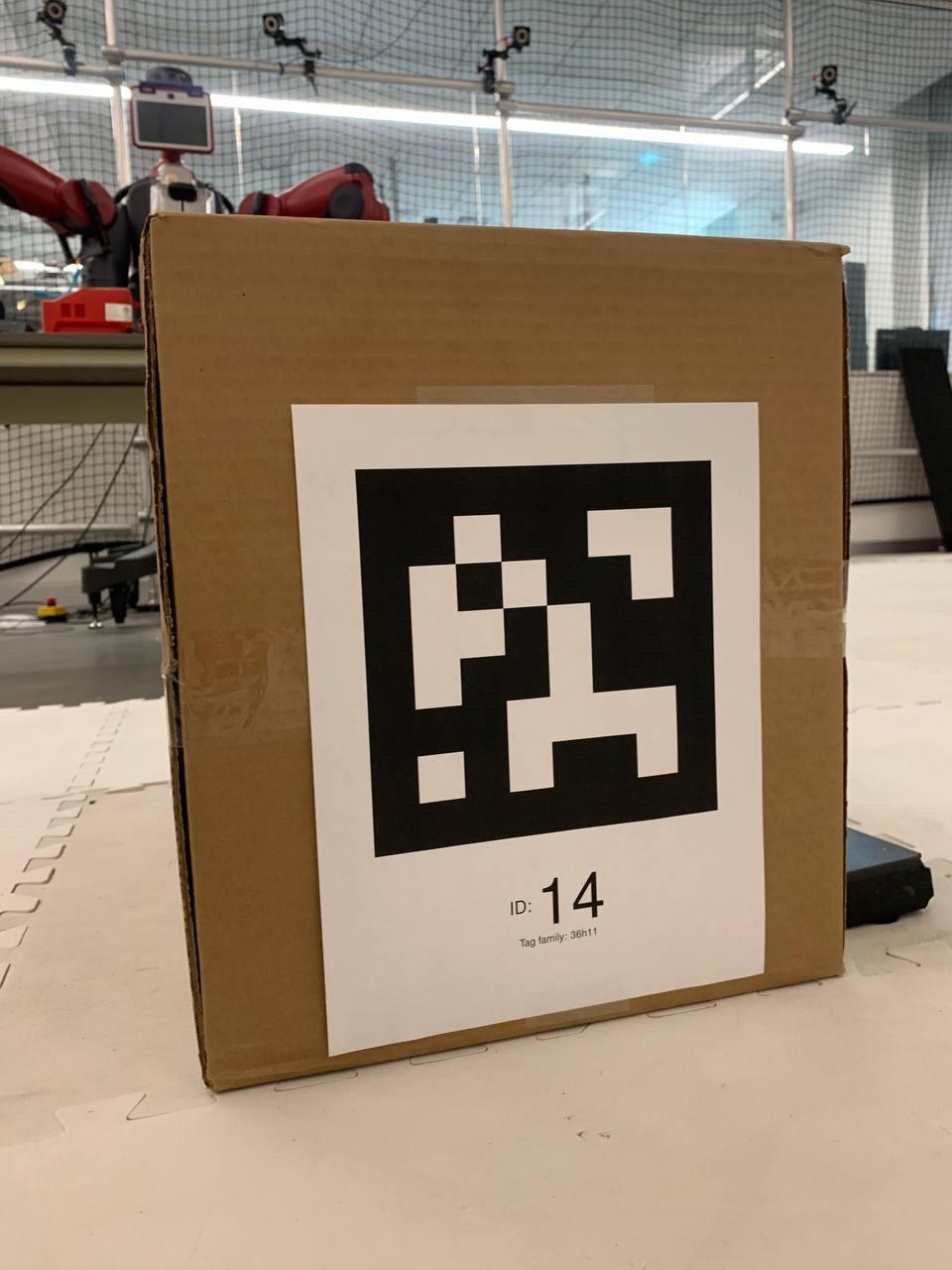} }}%
  \caption{The Jackal robot used for the experiments is shown in Fig.~\ref{fig:robot}. We use AprilTags (Fig.~\ref{fig:apriltag-figure}) as the landmarks for the algorithm.}
  \label{fig:}
\end{figure}

\begin{figure}
    \centering
    {\includegraphics[width=4cm]{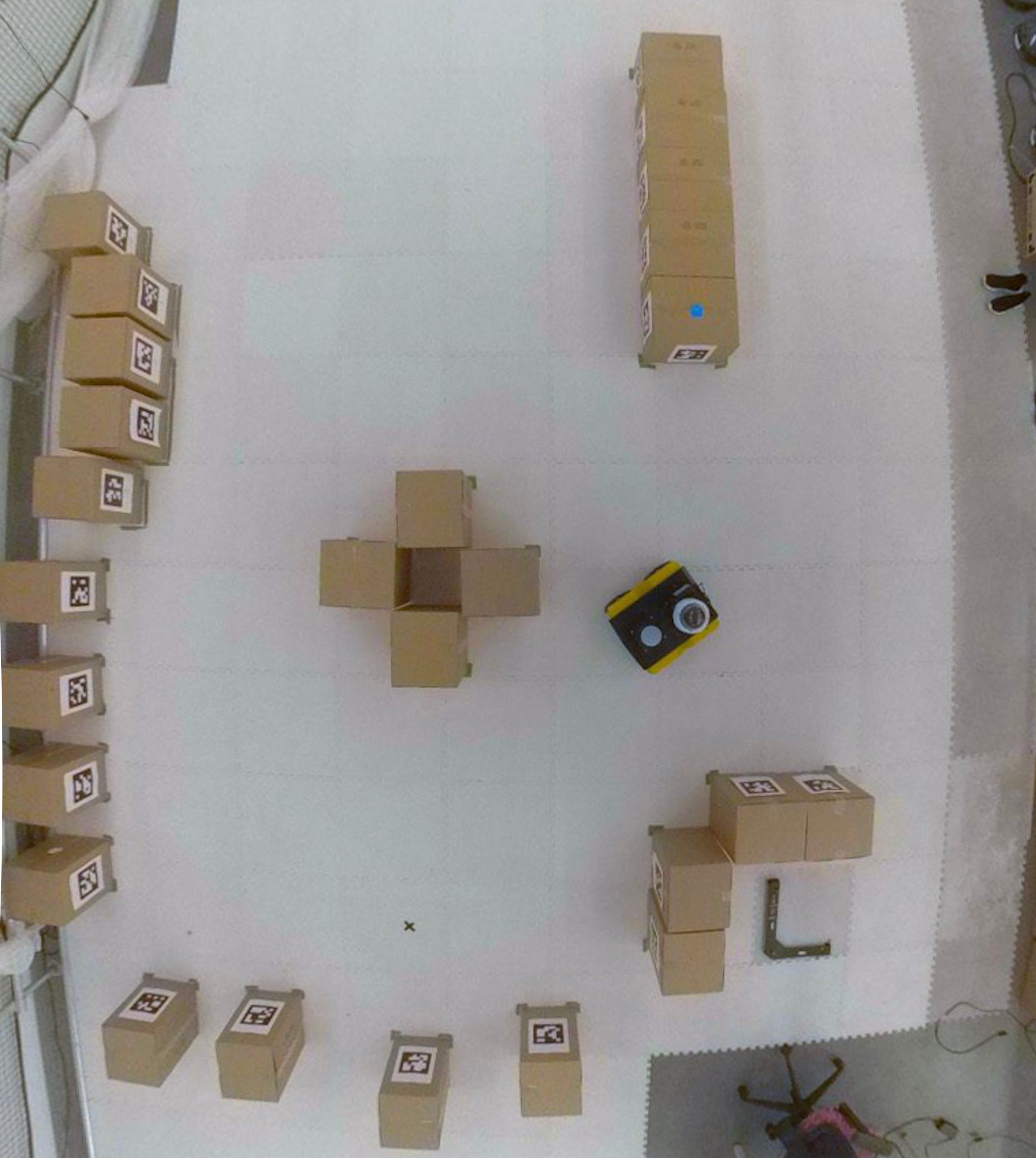} }
    \caption{Experimental Environment}
    \label{fig:original-env}
\end{figure}
\subsection{Clearpath Robotics Jackal Experiment}

We used the Jackal robot \cite{Clearpath} to test our algorithm in a lab environment. A bird's-eye view of the environment is presented in Fig.~\ref{fig:original-env}. We equipped the Jackal with an Intel RealSense D455 camera \cite{intel}. The Jackal's trajectory is tracked by the OptiTrack motion capture system (with 44 infra-red cameras), and we use AprilTags \cite{Wang2016} with unique codes as the landmark fiducials, which are placed in the environment at known positions and orientations with respect to the inertial reference frame of OptiTrack. The data was managed by the Robot Operating System (ROS) \cite{ROS}.

The experimental implementation of our algorithm requires taking into account several practical considerations. First, due to the limited field of view inherent in the RealSense camera, which prevents the Jackal from detecting all required fiducials, we use the proposed method in Sec.~\ref{sec:limited-field-of-view} to compute the controller based on two fiducials detected by the camera at each time instant.
Second, all bearing measurements described in previous sections are assumed to be rotationally aligned with the global inertial reference frame. 
To accommodate for this assumption, we compute the global rotation of the Jackal with respect to the world reference frame as
\begin{equation}
    ^W\mathbf{R}_{J} = {}^W\mathbf{R}_{AT} \left({}^{J}\mathbf{R}_{AT}\right)^{-1},
\end{equation}
where ${}^{J}\mathbf{R}_{AT}$ is the measured rotation of the AprilTag with respect to the Jackal.
Then, a normalized bearing ${}^J\mathbf{t}_{AT-J}$ detected by the Jackal can be aligned with the global reference frame as ${}^W\mathbf{t}_{AT-J} = {}^W\mathbf{R}_{J} {}^J\mathbf{t}_{AT-J}$.


Finally, previous sections assumed a linear dynamical model for the robot, while the Jackal has a unicycle dynamics. We map the original 2D input $u$ to a linear velocity $u_x$ along the $x$ axis of the robot and an angular velocity $\omega_z$ around the $z$ axis using a standard low-level controller:

\begin{align}
u_x = \small{\dfrac{\alpha}{\left\| u \right\|}
\begin{bmatrix}
    \cos{\varphi} \\ \sin{\varphi}
  \end{bmatrix}}\transpose u,&&
\omega_z = \small{\frac{\beta}{\left\| u \right\|} \bmat{0\\0\\1}\transpose \left(
                       \begin{bmatrix}
                         \cos{\varphi} \\ \sin{\varphi} \\ 0
                       \end{bmatrix}
  \times \begin{bmatrix}
    u \\ 0
  \end{bmatrix} \right)}
\end{align}
where $\varphi$ is the instantaneous yaw rotation of the robot with respect to the world reference frame, which is extracted from ${}^{C}\mathbf{R}_{AT}$ and ${}^{W}\mathbf{R}_{AT}$, and $\times$ represent the 3-D cross product; $\alpha$ and $\beta$ are user-defined scalar gains, 0.1 and 0.5 respectively.

We moved the Jackal across 4 different paths in the environment. 
The Jackal followed the edges of the simplified \rrtstar{} tree and reached the expected goal for all starting positions, despite the fact that the measurements were obtained with vision alone, and despite different dynamics of the robot. The video of the experiments is available at \url{https://github.com/Mahrooo/NavigationwithBearingMeasurements}.

\section{CONCLUSIONS AND FUTURE WORKS}\label{sec:conclusions}

In this work, we proposed a novel approach to synthesize a set of out-put feedback controllers on elements of a cell decomposition of the environment that is based on the tree generated by a simplified version of sampling-based \rrtstar{} method. We assumed the robot is equipped with a monocular camera and there exists a global compass. We introduced a new approach to modify the bearing measurements such that they form a uniform scaling of the set of landmarks and defined the output feedback controllers take as input the modified bearing measurements. Then, we computed the output-feedback controllers by implementing the robust Linear Programming that satisfied the stability of the system by the CLF constraint and satisfied the safety of the system by CBF constraint. We validated our approach on both the simulation and the experimental environment and represented the robustness of our approach to the mismatches in the dynamical model of the robot, and to the measurements acquired with a camera with a limited field of view. For the future,  we plan to prove to study the robustness of our algorithm to the deformation of the environment. In addition, we plan to find the output feedback controller when there exists an uncertainty in the bearing measurements.

\bibliographystyle{IEEEtran}
\bibliography{references.bib,biblio/IEEEFull.bib,biblio/planning.bib,biblio/hardware.bib,biblio/control.bib}

\begin{thebibliography}{10}
\providecommand{\url}[1]{#1}
\csname url@samestyle\endcsname
\providecommand{\newblock}{\relax}
\providecommand{\bibinfo}[2]{#2}
\providecommand{\BIBentrySTDinterwordspacing}{\spaceskip=0pt\relax}
\providecommand{\BIBentryALTinterwordstretchfactor}{4}
\providecommand{\BIBentryALTinterwordspacing}{\spaceskip=\fontdimen2\font plus
\BIBentryALTinterwordstretchfactor\fontdimen3\font minus
  \fontdimen4\font\relax}
\providecommand{\BIBforeignlanguage}[2]{{%
\expandafter\ifx\csname l@#1\endcsname\relax
\typeout{** WARNING: IEEEtran.bst: No hyphenation pattern has been}%
\typeout{** loaded for the language `#1'. Using the pattern for}%
\typeout{** the default language instead.}%
\else
\language=\csname l@#1\endcsname
\fi
#2}}
\providecommand{\BIBdecl}{\relax}
\BIBdecl

\bibitem{laumond1998robot}
J.-P. Laumond \emph{et~al.}, \emph{Robot motion planning and control}.\hskip
  1em plus 0.5em minus 0.4em\relax Springer, 1998, vol. 229.

\bibitem{latombe2012robot}
J.-C. Latombe, \emph{Robot motion planning}.\hskip 1em plus 0.5em minus
  0.4em\relax Springer Science \& Business Media, 2012, vol. 124.

\bibitem{choset2005principles}
H.~M. Choset, S.~Hutchinson, K.~M. Lynch, G.~Kantor, W.~Burgard, L.~E. Kavraki,
  and S.~Thrun, \emph{Principles of robot motion: theory, algorithms, and
  implementation}.\hskip 1em plus 0.5em minus 0.4em\relax MIT press, 2005.

\bibitem{kavraki1996probabilistic}
L.~E. Kavraki, P.~Svestka, J.-C. Latombe, and M.~H. Overmars, ``Probabilistic
  roadmaps for path planning in high-dimensional configuration spaces,''
  \emph{IEEE transactions on Robotics and Automation}, vol.~12, no.~4, pp.
  566--580, 1996.

\bibitem{rrt}
S.~M. LaValle, ``Rapidly-exploring random trees: A new tool for path
  planning,'' Iowa State University, Tech. Rep., 1998.

\bibitem{lavalle2001randomized}
S.~M. LaValle and J.~J. Kuffner~Jr, ``Randomized kinodynamic planning,''
  \emph{The international journal of robotics research}, vol.~20, no.~5, pp.
  378--400, 2001.

\bibitem{karaman2011sampling}
S.~Karaman and E.~Frazzoli, ``Sampling-based algorithms for optimal motion
  planning,'' \emph{The international journal of robotics research}, vol.~30,
  no.~7, pp. 846--894, 2011.

\bibitem{lavalle2006planning}
S.~M. LaValle, \emph{Planning algorithms}.\hskip 1em plus 0.5em minus
  0.4em\relax Cambridge university press, 2006.

\bibitem{renganathan2020towards}
V.~Renganathan, I.~Shames, and T.~H. Summers, ``Towards integrated perception
  and motion planning with distributionally robust risk constraints,''
  \emph{arXiv preprint arXiv:2002.02928}, 2020.

\bibitem{Mahroo}
M.~Bahreinian, E.~Aasi, and R.~Tron, ``Robust planning and control for
  polygonal environments via linear programming,'' \emph{2020 IEEE American
  Control Conference (ACC)}, 2020.

\bibitem{Mahroo2}
M.~Bahreinian, M.~Mitjans, and R.~Tron, ``Robust sample-based output-feedback
  path planning,'' in \emph{2021 IEEE/RSJ International Conference on
  Intelligent Robots and Systems (IROS)}.\hskip 1em plus 0.5em minus
  0.4em\relax IEEE, pp. 5780--5787.

\bibitem{zhao2011motion}
H.~Zhao and Z.~Wang, ``Motion measurement using inertial sensors, ultrasonic
  sensors, and magnetometers with extended kalman filter for data fusion,''
  \emph{IEEE Sensors Journal}, vol.~12, no.~5, pp. 943--953, 2011.

\bibitem{papanikolopoulos1993visual}
N.~P. Papanikolopoulos, P.~K. Khosla, and T.~Kanade, ``Visual tracking of a
  moving target by a camera mounted on a robot: A combination of control and
  vision,'' \emph{IEEE transactions on robotics and automation}, vol.~9, no.~1,
  pp. 14--35, 1993.

\bibitem{yata1999fast}
T.~Yata, A.~Ohya, and S.~Yuta, ``A fast and accurate sonar-ring sensor for a
  mobile robot,'' in \emph{Proceedings 1999 IEEE International Conference on
  Robotics and Automation (Cat. No. 99CH36288C)}, vol.~1.\hskip 1em plus 0.5em
  minus 0.4em\relax IEEE, 1999, pp. 630--636.

\bibitem{lumelsky1990unified}
V.~Lumelsky and K.~Sun, ``A unified methodology for motion planning with
  uncertainty for 2d and 3d two-link robot arm manipulators,'' \emph{The
  International journal of robotics research}, vol.~9, no.~5, pp. 89--104,
  1990.

\bibitem{lambrinos1998landmark}
D.~Lambrinos, R.~M{\"o}ller, R.~Pfeifer, and R.~Wehner, ``Landmark navigation
  without snapshots: the average landmark vector model,'' in \emph{Proc.
  Neurobiol. Conf. G{\"o}ttingen}, 1998.

\bibitem{argyros2001robot}
A.~A. Argyros, K.~E. Bekris, and S.~C. Orphanoudakis, ``Robot homing based on
  corner tracking in a sequence of panoramic images,'' in \emph{Proceedings of
  the 2001 IEEE Computer Society Conference on Computer Vision and Pattern
  Recognition. CVPR 2001}, vol.~2.\hskip 1em plus 0.5em minus 0.4em\relax IEEE,
  2001, pp. II--II.

\bibitem{liu2010bearing}
M.~Liu, C.~Pradalier, Q.~Chen, and R.~Siegwart, ``A bearing-only 2d/3d-homing
  method under a visual servoing framework,'' in \emph{2010 IEEE International
  Conference on Robotics and Automation}.\hskip 1em plus 0.5em minus
  0.4em\relax IEEE, 2010, pp. 4062--4067.

\bibitem{tron2014optimization}
R.~Tron and K.~Daniilidis, ``An optimization approach to bearing-only visual
  homing with applications to a 2-d unicycle model,'' in \emph{2014 IEEE
  International Conference on Robotics and Automation (ICRA)}.\hskip 1em plus
  0.5em minus 0.4em\relax IEEE, 2014, pp. 4235--4242.

\bibitem{pentland1989simple}
A.~Pentland, T.~Darrell, M.~Turk, and W.~Huang, ``A simple, real-time range
  camera,'' in \emph{1989 IEEE Computer Society Conference on Computer Vision
  and Pattern Recognition}.\hskip 1em plus 0.5em minus 0.4em\relax IEEE
  Computer Society, 1989, pp. 256--257.

\bibitem{ens1993investigation}
J.~Ens and P.~Lawrence, ``An investigation of methods for determining depth
  from focus,'' \emph{IEEE Transactions on pattern analysis and machine
  intelligence}, vol.~15, no.~2, pp. 97--108, 1993.

\bibitem{michels2005high}
J.~Michels, A.~Saxena, and A.~Y. Ng, ``High speed obstacle avoidance using
  monocular vision and reinforcement learning,'' in \emph{Proceedings of the
  22nd international conference on Machine learning}, 2005, pp. 593--600.

\bibitem{Choset:book05}
H.~M. Choset, K.~M. Lynch, S.~Hutchinson, G.~Kantor, W.~Burgard, L.~Kavraki,
  and S.~Thrun, \emph{Principles of robot motion: theory, algorithms, and
  implementation}.\hskip 1em plus 0.5em minus 0.4em\relax MIT press, 2005.

\bibitem{Isidori:book95}
A.~Isidori, \emph{Nonlinear control systems}.\hskip 1em plus 0.5em minus
  0.4em\relax Springer Science \& Business Media, 1995.

\bibitem{positiveInvariant}
F.~Borrelli, A.~Bemporad, and M.~Morari, \emph{Predictive control for linear
  and hybrid systems}.\hskip 1em plus 0.5em minus 0.4em\relax Cambridge
  University Press, 2017.

\bibitem{zcbf1}
X.~Xu, P.~Tabuada, J.~W. Grizzle, and A.~D. Ames, ``Robustness of control
  barrier functions for safety critical control,'' \emph{IFAC-PapersOnLine},
  vol.~48, no.~27, pp. 54--61, 2015.

\bibitem{nguyen2016exponential}
Q.~Nguyen and K.~Sreenath, ``Exponential control barrier functions for
  enforcing high relative-degree safety-critical constraints,'' in \emph{2016
  American Control Conference (ACC)}.\hskip 1em plus 0.5em minus 0.4em\relax
  IEEE, 2016, pp. 322--328.

\bibitem{Clearpath}
Clearpath, ``Jackal unmanned ground vehicle,''
  \url{https://clearpathrobotics.com/jackal-small-unmanned-ground-vehicle/},
  accessed: 2021-02-26.

\bibitem{intel}
intel, ``Intel realsense d455,''
  \url{https://www.intelrealsense.com/depth-camera-d455/}.

\bibitem{Wang2016}
J.~Wang and E.~Olson, ``{AprilTag 2: Efficient and robust fiducial
  detection},'' in \emph{2016 IEEE/RSJ International Conference on Intelligent
  Robots and Systems (IROS)}.\hskip 1em plus 0.5em minus 0.4em\relax IEEE, oct
  2016, pp. 4193--4198.

\bibitem{ROS}
{Stanford Artificial Intelligence Laboratory et al.}, ``Robotic operating
  system, melodic morenia,'' https://www.ros.org.

\end{thebibliography}
\end{document}